\documentclass[sigplan]{acmart}
\usepackage[utf8]{inputenc} %
\usepackage[T1]{fontenc}    %
\usepackage{hyperref}       %
\usepackage{url}            %
\usepackage{booktabs}       %
\usepackage{nicefrac}       %
\usepackage{microtype}      %
\usepackage{xcolor}         %
\usepackage{cleveref}

\usepackage{mathtools}
\usepackage{siunitx}
\usepackage[abbreviations]{foreign}
\usepackage{multirow}
\usepackage{color, colortbl}
\usepackage{bm}

\usepackage[ruled,vlined]{algorithm2e}
\SetKwComment{Comment}{$\vartriangleright$ }{}

\DeclareMathOperator{\argmin}{argmin}%

\usepackage{graphicx}
\usepackage{subcaption}

\newboolean{showcomments}
\setboolean{showcomments}{true}
\ifthenelse{\boolean{showcomments}}
{ \newcommand{\mynote}[3]{
		\fbox{\bfseries\sffamily\scriptsize#1}
		{\small$\blacktriangleright$\textsf{\emph{\color{#3}{#2}}}$\blacktriangleleft$}}
	\newcommand{\zzz}[1]{{\setlength{\fboxsep}{2pt}\fcolorbox{black}{yellow}{\textsf{\emph{#1}}}}\xspace}}
{ \newcommand{\mynote}[3]{}
	\newcommand{\zzz}[1]{}}

\usepackage{acronym}
\acrodef{DL}{deep learning}
\acrodef{ML}{machine learning}
\acrodef{FL}{federated learning}
\acrodef{SGD}{stochastic gradient descent}
\acrodef{IID}{independent and identically distributed}
\acrodef{non-IID}{non independent and identically distributed}
\acrodef{HPO}{hyperparameter optimization}
\acrodef{HP}{hyperparameter}
\acrodef{KD}{knowledge distillation}

\newcommand{\sys}{\textsc{CPFL}\xspace}
\newcommand{\KD}{\ac{KD}\xspace}
\newcommand{\FL}{\ac{FL}\xspace}

\newcommand{\femnist}{FEMNIST\xspace}

\newcommand{\cifar}{CIFAR-10\xspace}

\newcommand{\fedkd}{FedKD\xspace}
\newcommand{\fedavg}{FedAvg\xspace}

\newcommand{\vect}[1]{\ensuremath{\bm{#1}}}
\newcommand{\E}{\ensuremath{\mathbb{E}}}

\newcommand{\Hcal}{\ensuremath{\mathcal{H}}}
\newcommand{\Dcal}{\ensuremath{\mathcal{D}}}
\newcommand{\Lcal}{\ensuremath{\mathcal{L}}}

\newcommand{\Rb}{\ensuremath{\mathbb{R}}}
\newcommand{\Ycal}{\ensuremath{\mathcal{Y}}}
\newcommand{\Xcal}{\ensuremath{\mathcal{X}}}
\newcommand{\Zcal}{\ensuremath{\mathcal{Z}}}

\newcommand{\parenthese}[1]{\left(#1\right)}
\newcommand{\bracket}[1]{\left[#1\right]}

\newcommand{\curlybracket}[1]{\left\{#1\right\}}

\newcommand{\x}{\boldsymbol{x}\xspace}

\usepackage{tikz}
\usepackage{pgfplots}
\usepackage{sansmath}
\pgfplotsset{compat=newest}

\usepgfplotslibrary{external,units,colorbrewer,groupplots,fillbetween}
\tikzsetexternalprefix{figures/}
\tikzset{external/mode=list and make}
\usetikzlibrary{patterns}

\makeatletter
\begingroup\endlinechar=-1\relax
\everyeof{\noexpand}%
\edef\x{\endgroup\def\noexpand\homepath{%
		\@@input|"kpsewhich --var-value=HOME" }}\x
\makeatother

\def\overleafhome{/tmp}
\newcommand{\inputplot}[2]{%
	\ifx\homepath\overleafhome%
	\IfBeginWith{#1}{plots}{\includegraphics{main-figure#2.pdf}}{#1}%
	\else%
	{\sffamily\scriptsize\input{#1}}
	\fi}

\newcommand{\newgroupwidth}[2]%
{\expandafter\xdef\csname groupwidth#1\endcsname{#2}}

\newcounter{groupwidth}
\newsavebox{\groupwidthbox}
\makeatletter
\newenvironment{groupwidth}[1]%
{\edef\groupnumber{#1}%
	\stepcounter{groupwidth}%
	\@ifundefined{groupwidth\thegroupwidth}{\pgfmathsetlengthmacro{\mywidth}{(\linewidth-0cm)/\groupnumber}}%
	{\expandafter\let\expandafter\mywidth\csname groupwidth\thegroupwidth\endcsname}%
	\begin{lrbox}{\groupwidthbox}%
		\tikzset{/pgfplots/width={\mywidth}}%
		\ignorespaces}%
	{\end{lrbox}%
	\usebox\groupwidthbox
	\pgfmathsetlengthmacro{\mywidth}{\mywidth + ((\linewidth-0cm) - \wd\groupwidthbox)/\groupnumber}
	\immediate\write\@auxout{\string\newgroupwidth{\thegroupwidth}{\mywidth}}}
\makeatother

\usepackage{physics}
\usepackage{enumitem}
\usepackage{thmtools} 
\usepackage{thm-restate}
\usepackage{bbm}
\theoremstyle{definition}

\theoremstyle{remark}

\newtheorem{theorem}{Theorem}
\newtheorem{lemma}[theorem]{Lemma}

\allowdisplaybreaks

\newcommand{\R}{\mathbb{R}}

\graphicspath{ {figures/} }

\copyrightyear{2025}
\acmYear{2025}
\setcopyright{acmlicensed}\acmConference[EuroMLSys '25]{The 5th Workshop on Machine Learning and Systems }{March 30-April 3 2025}{Rotterdam, Netherlands}
\acmBooktitle{The 5th Workshop on Machine Learning and Systems (EuroMLSys '25), March 30-April 3 2025, Rotterdam, Netherlands}
\acmDOI{10.1145/3721146.3721939}
\acmISBN{979-8-4007-1538-9/2025/03}


\begin{document}

\title{Harnessing Increased Client Participation with Cohort-Parallel Federated Learning}

\author{Akash Dhasade}
\affiliation{
  \institution{EPFL}
  \city{Lausanne}
  \country{Switzerland}
}

\author{Anne-Marie Kermarrec}
\affiliation{
  \institution{EPFL}
  \city{Lausanne}
  \country{Switzerland}
}

\author{Tuan-Anh Nguyen}
\affiliation{
  \institution{Independent Researcher}
  \city{Grenoble}
  \country{France}
}

\author{Rafael Pires}
\affiliation{
  \institution{EPFL}
  \city{Lausanne}
  \country{Switzerland}
}

\author{Martijn de Vos}
\affiliation{
  \institution{EPFL}
  \city{Lausanne}
  \country{Switzerland}
}

\renewcommand{\shortauthors}{Dhasade et al.}

\begin{abstract}
Federated learning (FL) is a machine learning approach where nodes collaboratively train a global model.
As more nodes participate in a round of FL, the effectiveness of individual model updates by nodes also diminishes.
In this study, we increase the effectiveness of client updates by dividing the network into smaller partitions, or \emph{cohorts}.
We introduce Cohort-Parallel Federated Learning (CPFL): a novel learning approach where each cohort independently trains a global model using FL, until convergence, and the produced models by each cohort are then unified using knowledge distillation.
The insight behind CPFL is that smaller, isolated networks converge quicker than in a one-network setting where all nodes participate.
Through exhaustive experiments involving realistic traces and non-IID data distributions on the CIFAR-10 and FEMNIST image classification tasks, we investigate the balance between the number of cohorts, model accuracy, training time, and compute resources.
Compared to traditional FL, CPFL with four cohorts, non-IID data distribution, and CIFAR-10 yields a 1.9$\times$ reduction in train time and a 1.3$\times$ reduction in resource usage, with a minimal drop in test accuracy.
\end{abstract}


\maketitle

\acresetall
\section{Introduction}

\Ac{FL} allows for the collaborative training of a machine learning model across a distributed network of training nodes, or clients, without ever moving training data~\cite{mcmahan2017communication}.
A central server orchestrates the process by selecting a subset of clients, referred to as a \emph{cohort}~\cite{charles2021large}, and sends them the most recent version of the global model.
Subsequently, clients in this cohort undertake a few training steps on their local datasets, contributing to the refinement of the model.
The locally updated models are then transmitted back to the server for aggregation.
This iterative process continues with the server selecting another cohort, possibly composed of a different set of clients, for each successive training round until the global model converges.

The effect of the cohort size on \ac{FL} performance has been assessed in numerous studies~\cite{mcmahan2017communication,charles2021large,huba2022papaya,charles2024towards}.
While larger cohort sizes intuitively learn from more data in each round, thus accelerating the convergence of the global model~\cite{azam2023federated}, they have been found to yield diminishing returns~\cite{charles2021large}.
Furthermore, larger cohorts often use client updates inefficiently, requiring more resources and time to reach similar accuracy levels compared to smaller cohort sizes~\cite{charles2021large, huba2022papaya}. 
As a result, current methods struggle to fully take advantage of model updates from a large number of clients~\cite{bonawitz2019towards}. 

This work explores a strategy that harnesses increased client participation more efficiently.
\Ac{FL} research traditionally focuses on the one-cohort setting and speeds up model convergence by increasing the size of a cohort~\cite{bonawitz2019towards}.
We instead propose and investigate the simple idea of \textit{partitioning} the network into several cohorts\footnote{We refer to network partitions as \emph{cohorts} throughout this work.}, each of which runs independent and parallel FL training sessions.
We name this approach \textit{Cohort-Parallel Federated Learning}, or \sys.
The architecture of \sys is illustrated in~\Cref{fig:architecture}.
Clients within a cohort contribute to training an \ac{FL} model specific to that cohort (step 1), leveraging any existing \ac{FL} algorithm~\cite{abdelmoniem2023refl,lai2021oort,mcmahan2017communication}.
This trained cohort model is then uploaded to the global FL server (step 2).
Finally, the server unifies these multiple cohort models into a single global model through the process of \Ac{KD} (step 3).
\ac{KD} is a technique that combines the knowledge of different teacher models into a unified student model~\cite{hintondistillation}.
We leverage cross-domain, unlabeled and public datasets to carry out \ac{KD} and produce the final model.

\begin{figure}[t]
    \centering
    \includegraphics[scale=0.9]{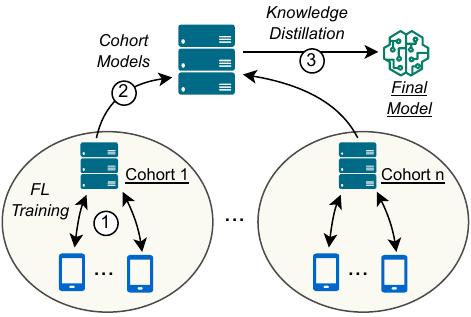}
    \caption{The architecture of Cohort-Parallel Federated Learning (\sys).}
    \label{fig:architecture}
\end{figure}

\sys provides three main benefits compared to single-cohort FL.
Firstly, smaller networks make more efficient use of client updates, \textit{reducing the computation and communication resources} required to train a model.
Secondly, smaller networks \textit{converge significantly quicker} than the entire network, reducing overall training time.
We illustrate this effect in Figure~\ref{fig:motivation}, which shows the evolution of validation loss without partitioning the network (dashed curve) and when partitioning the network (solid curves).
We annotate with vertical lines the round when models have converged, according to our stopping criterion described in~\Cref{sec:exp_setup}.
From this figure, it is evident that smaller partitions converge faster than when the network is not partitioned, both with IID and non-IID data distributions.
Thirdly, partitioning the network into smaller cohorts provides a flexible means to control resource usage and time to convergence by appropriately choosing number of cohorts.
Our proposal is generic and can therefore be applied to any \ac{FL} setting.

\paragraph{Contributions.} This work makes the following contributions:
\begin{itemize}
    \item We propose \sys, which uses partitioning as a simple and effective strategy to improve FL efficiency in terms of time-to-accuracy and training resource usage (\Cref{sec:cpfl}). 
    \item We provide theoretical guarantees on the performance of the global model through domain adaptation analysis. Our result extends \KD to a more general setting where the teacher model is composed of a mixture of distributions, as is the case for \sys (\Cref{sec:da}).
    \item We conduct extensive experiments using realistic traces of devices exhibiting different compute and network speeds on two image classification datasets with varying data distributions (\Cref{sec:eval}). 
    We analyze the effect of the number of cohorts in \sys on the achieved test accuracy, resource utilization, and training time.
    Our results on the \cifar dataset under non-IID data distributions demonstrate that employing just four cohorts can already lead to a $1.3\times$ reduction in training resource usage and $1.9\times$ faster convergence with a minimal drop in test accuracy of 0.6\% compared to traditional \ac{FL}.
\end{itemize}

In summary, \sys offers \ac{FL} practitioners a simple and pragmatic method to obtain considerable resource savings and shorter training sessions in FL systems.

\begin{figure} [t]
	\inputplot{plots/motivation}{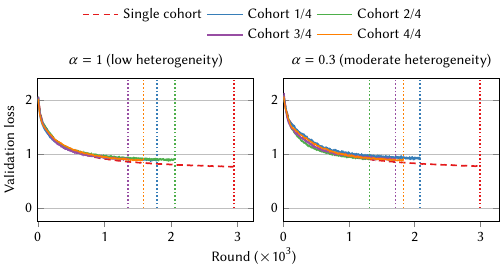}
	\caption{Comparing validation loss for partitioned (solid curves) and unpartitioned (dashed curve) networks across IID (left figure) and non-IID (right figure) distributions. The vertical dotted line denotes the convergence point of the training. Additional details on the experiment setup are provided in~\Cref{sec:setup_motivation_plot}.
    }
	\label{fig:motivation}
\end{figure}

\section{Related Work}
\label{sec:rw}

\textbf{Cohorts in Federated Learning.} 
The idea of grouping clients \textit{under some criterion} referred to as \textit{clustering} has been well-studied in the \FL literature~\cite{ghosh2020efficient,duan2021flexible,liu2022auxo}.
The first set of these works leverages clustering to solve the federated multi-task learning (FMTL) problem, which assumes that there exist $k$ different data distributions $\Dcal_1, \Dcal_2, \ldots \Dcal_K$ in a network of $m$ clients.
Each distribution $\Dcal_i$ corresponds to a different task $i$ where these approaches aim to cluster clients that solve similar tasks without the explicit knowledge of cluster identities~\cite{ghosh2020efficient,sattler2020clustered}.
The second set of works focuses on mitigating the impact of data \ac{non-IID}ness or client data distribution shift (\eg concept shift) by clustering statistically similar clients~\cite{briggs2020,liu2022auxo}. 
Such clustering relies on some statistics collected during training \eg similarity of gradient or model updates obtained from clients, and typically introduces overheads to obtain clusters.
Other works use clustering to build a multi-tier hierarchical topology considering both data distribution and communication efficiency~\cite{ma2024fractal}.
In contrast to the above works, we study and showcase the benefits of \textit{ partitioning} \ie simply dividing the network into arbitrary groups of clients.
Uniquely, our work explores the trade-offs between time to convergence, resource usage, and attained accuracies intricately tied to the number of partitions.

\paragraph{Parallelism in Federated Learning.}
The parallelism induced by increasing the number of participating clients per round was initially investigated by McMahan \etal \shortcite{mcmahan2017communication}.
They observed that leveraging more clients per round reduced the number of rounds required to reach a target accuracy. However, the benefits extended only to a certain threshold number of clients, beyond which the returns diminished significantly. 
Charles \etal \shortcite{charles2021large} further examined the impact of large client participation across various learning tasks.
The conclusions were similar, where the threshold was empirically shown to lie between 10 and 50, interestingly the same for all tasks.
In a network of thousands of clients with several hundred available clients per round, this empirical threshold clearly illustrates the limitations of current methods in effectively utilizing increased client participation.
\cite{charles2021large} also show that federated algorithms under large client participation use local updates inefficiently, requiring significantly more samples per unit-accuracy.
These limitations call for a novel approach that can better capitalize increased client participation~\cite{bonawitz2019towards}.

\paragraph{Knowledge Distillation (KD)}
In this work, we leverage \ac{KD} to combine the knowledge of individual cohort models into a single global model at the server.
\ac{KD} was initially proposed to extract information from a complex teacher model into a small student model~\cite{hintondistillation}. 
Conventionally, the training of the student model involves the minimization of the disparity between the logits produced by the student and teacher, which are computed utilizing an appropriate auxiliary dataset~\cite{hintondistillation}.
\ac{KD} has been increasingly used in \ac{FL} to reduce communication costs~\cite{he2020group,sattlercfd,gong2022preserving}, enable heterogeneous client models~\cite{li2019fedmd,taolin} or mitigate privacy risks~\cite{gong2022preserving}.
More generally, \ac{KD} has also been used to exchange knowledge in a distributed network of clients~\cite{bistritz2020distributed,zhmoginov2023decentralized}.

\paragraph{\sys extremes.} We highlight two algorithms in the literature that are special cases of our proposed \sys algorithm.
One extreme is the one-shot  \fedkd algorithm~\cite{gong2022preserving} wherein each node or client is its own cohort.  
On the other extreme is the standard \fedavg algorithm~\cite{mcmahan2017communication} where all nodes belong to a single cohort.
In particular, \fedavg does not utilize \ac{KD} since there is only one cohort, requiring no knowledge fusion.  
\sys can proficiently navigate this spectrum, exhibiting interesting characteristics that can be controlled by altering the number of cohorts.

\section{Cohort-Parallel Federated Learning}
\label{sec:cpfl}

\begin{algorithm}[t]
\DontPrintSemicolon
\SetAlgoLined
\KwIn{Set of $n$ cohorts $\curlybracket{\mathcal{C}_i}_{i=1}^n$, aggregation weights $\{p_i\}_{i=1}^n$, learning rate $\eta$, public dataset $\hat{\Dcal}_p$, parametric function $f$ from $h_{\theta}$, number of epochs $E$}
\KwOut{Final global model parameters $\theta$} 
\BlankLine
\textbf{Cohort Servers Execute}\;
\For{each cohort $i = 1, \ldots, n$ in parallel}{
  $\theta_i \gets$ train model using standard \FL \;
  Send $\theta_i$ to the global server \;

  }
\BlankLine
\textbf{Global Server Executes}\;
\text{Initialize global model parameters $\theta_s$}\;
\For{$\vect{x} \in \hat{\Dcal}_p$}{
        \For{$i = 1,\ldots,n$}{
        $\vect{z}_i \gets f(\vect{x};\theta_i)$ \;
         }
        $\widetilde{\vect{z}} \leftarrow \sum_{i=1}^n p_i \vect{z}_i $ \Comment*[r]{aggregate logits}
    }
\BlankLine
\For{$e = 1, \ldots , E$}{ 
    \For{\text{mini-batch} $\vect{b} \in \hat{\Dcal}_p$}{
        $\vect{z}_s \leftarrow f(\vect{b};\theta_s)$ \;
        $\theta_s \gets \theta_s - \eta \cdot \nabla_{\theta_s} \mathcal{L}$ \Comment*[r]{\ac{KD} using eq. (3)}
    }
}
\Return{$\theta$} \;
\caption{Cohort-Parallel Federated Learning}
\label{algo:cohort-parallel}
\end{algorithm}

We now describe \sys, a learning approach that combines the strength of multiple parallel \FL sessions with \ac{KD}.
We illustrate our approach in~\Cref{fig:architecture} and provide pseudo code in Algorithm~\ref{algo:cohort-parallel}.

\subsection{Algorithm overview}
Consider a supervised learning problem with input space $\mathcal{X}$ and output space $\mathcal{Y}$. 
For a model $h_\theta: \Xcal \rightarrow \Zcal$ parameterized by $\theta \in \R^d$, each data point $(x,y) \in \Xcal \times \Ycal$ incurs a loss of $\ell(h_\theta(x),y)$ where $\ell: \Zcal \times \Ycal \to \R$ is a non-negative loss function.
The expected loss of $h_{\theta}$ on data distribution $\Dcal$ is defined as $\Lcal_{\Dcal} (\theta) = \E_{(\vect{x},y) \sim \Dcal} \bracket{\ell \parenthese{ h_{\theta}(\vect{x}), y}}$. 
We consider an \FL setting with $M$ clients and our algorithm proceeds as follows.

In the first stage of the algorithm, the global server randomly partitions the clients into $n$ cohorts, each comprising $K$ clients such that $n*K=M$.\footnote{This is for simplicity - our setting still holds otherwise.}
We opt for a random partitioning of nodes into cohorts due to its simplicity and universality. 
This approach ensures unbiased division without introducing added complexities or biases that more advanced partitioning strategies might incur\footnote{An advanced partitioning scheme could be a multi-objective function that considers data skewness, device characteristics, \etc.}.
Each client has a private local dataset following its local distribution $\Dcal_{i,k}$.
Cohorts operate in parallel, running independent \FL training sessions using any traditional \FL algorithm, such as \fedavg, while reporting to their respective cohort servers, as shown in Step 1 of~\Cref{fig:architecture}. 
These servers could, for example, correspond to geographically distributed \FL servers running \FL sessions within their geographic boundaries or correspond to co-located servers within a single global server.
Each such server trains cohort specific parameters $\theta_i$, which minimizes the average risk across all clients within the cohort. 
Precisely, for every $i \in [n]$, 
\begin{align}
	\label{eq:fl-definition}
	\theta_i \in \argmin_{\theta} \frac{1}{K} \sum_{k=1}^K \Lcal_{\Dcal_{i,k}} (\theta).
\end{align}
Once training converges, the cohort server transmits \( \theta_i \) to the global server, as shown in Step 2 of~\Cref{fig:architecture}.  

In the second stage of the algorithm, the global server distills the knowledge of each cohort model (called teacher models) into a single global model (called the student model) once all cohorts have converged, shown in step 3 in~\Cref{fig:architecture}.
This knowledge transfer is facilitated by an unlabeled public dataset $\hat{\Dcal}_p$.
Specifically, the global server generates logit vectors $\vect{z}_i = f(\vect{x}; \theta_i)$ for every $\vect{x} \in \hat{\Dcal}_p$ for each cohort model $i \in [n]$.
These logits are then aggregated $\widetilde{\vect{z}} := \sum_{i=1}^n p_i \vect{z}_i$ where $p_i$ denotes the weights of the aggregation with $\sum_{i=1}^n p_i = 1$ to act as soft-targets for the process of knowledge distillation.
Denoting by $\theta_s$ the parameters of the global server's model and $\vect{z}_s = f(\vect{x}; \theta_s)$, the global server solves the following optimization problem:
\begin{align}
	\theta_s \in \argmin_{\theta} \E_{\vect{x} \sim \Dcal_p} \bracket{ 	\Lcal(\vect{z}_s,  \widetilde{\vect{z}}) } \label{eq:dist_loss_all} \\
	\text{where } \Lcal(\vect{z}_s,  \widetilde{\vect{z}}) = \norm{\vect{z}_s - \widetilde{\vect{z}}}_1 \label{eq:dist_loss}
\end{align} 
and $\norm{.}_1$ represents the L1 norm. Our complete approach is outlined in Algorithm~\ref{algo:cohort-parallel}.

We set the weights $\{p_i\}_{i=1}^n$ for logit aggregation by extending the approach used in one-shot \fedkd \cite{gong2022preserving} based on the label distribution of clients.
Each cohort server first aggregates the label distributions of its nodes, forming a cohort-wide distribution.  
The global server then assigns per cohort weight based on its aggregated label distribution.
Compared to unweighted averaging, this significantly boosts the effectiveness of knowledge distillation since particular cohorts might be better suited to predict particular target classes~\cite{gong2022preserving}.
However, sharing of label distributions by nodes within a cohort may pose privacy risks.
To mitigate this risk, one can compute aggregated label distributions using secure hardware~\cite{dhasade2022tee} or via secure aggregation~\cite{10.1145/3133956.3133982}, avoiding the leakage of individual label distributions.
Lastly, we highlight that \sys can leverage any existing \FL algorithm in stage 1, ranging from \fedavg~\cite{mcmahan2017communication} to more advanced algorithms~\cite{lai2021oort,abdelmoniem2023refl}.

\subsection{Cross-domain analysis}
\label{sec:da}
In line with our setting, we established a bridge between knowledge distillation and the theory of domain adaptation. The concept of domain adaptation revolves around training a classifier to perform effectively in a target domain using a model previously trained in a source domain. This concept mirrors the structure of our framework, where we distill knowledge from models trained in parallel using FL on multiple distinct distributions within a cohort and subsequently transfer it to a global model. By viewing our framework through this lens, we developed a generalization bound for our distilled model by drawing upon the principles of domain adaptation theory \cite{domain-adaptation-theory}.

Let $\Dcal$ represent the target distribution and $\Hcal$ be the hypothesis class parameterized by $\Theta \subset \Rb^d$ as  $\Hcal = \curlybracket{h_{\theta}, \theta \in \Theta}$. 
With a slight change in notation, we use $h$ in place of $h_{\theta}$ for simplicity and indicate by $\Lcal_{\Dcal}(h)$ the risk of $h \in \Hcal$. 
We define the $\Hcal \Delta \Hcal$-divergence between a source distribution $\Dcal'$ and the target distribution $\Dcal$, $d_{\mathcal{H} \Delta \mathcal{H}} (\Dcal', \Dcal)$ as 
\begin{align*}
 2 \sup_{h,h' \in \mathcal{H}} \abs{ \mathbb{P}_{\vect{\vect{x}} \sim \Dcal'} (h(\vect{x}) \neq h'(\vect{x})) - \mathbb{P}_{\vect{x} \sim \Dcal} (h(\vect{x}) \neq h'(\vect{x}))}
\end{align*}
with $\Hcal \Delta \Hcal$ representing the symmetric difference space defined as $\Hcal \Delta \Hcal := \curlybracket{h(\vect{x}) \oplus h'(\vect{x}) \vert h,h' \in \Hcal}$. 
Furthermore, we define $ \lambda := \inf_{h \in \Hcal} \curlybracket{\Lcal_{\Dcal'} (h) + \Lcal_{\Dcal}(h)} $ as the risk for the optimal hypothesis across the two distributions.

For each cohort $i \in [n]$, we designate $h_i$ and $\Dcal_i$ as the respective hypothesis and distribution. Given that each cohort conducts federated training in parallel, we characterize the distribution $\Dcal_i$ as a mixture of distributions from its $K$ clients \ie $\Dcal_i = \frac{1}{K} \sum_{k=1}^K \Dcal_{i,k}$. Consequently, we consider a problem of $n*K$ sources domain adaptation. Denoting the hypothesis on the global server by $h_s = \sum_{i=1}^n p_i h_i $, we state the following theorem:
\begin{theorem}
\label{thm:thm1}
Let $\Hcal$ be a finite hypothesis class and $h_s :=\sum_{i=1}^n p_i h_i$, where $p_i > 0$ and $\sum_{i=1}^n p_i = 1$. Suppose that each source dataset has $m$ instances. Then, for any $\delta \in (0,1)$, with probability at least $1-\delta$, the expected risk of $h_s$ on the target distribution $\Dcal$ is bounded by:
\begin{equation}
\label{eq:theorem_bound}
    \begin{aligned}
        \mathcal{L}_{\Dcal} (h_s)
            &\leq \sum_{i=1}^n \sum_{k=1}^K \frac{p_i}{K} \mathcal{L}_{\hat{\Dcal}_{i,k}} (h_{i}) \\
            &\quad + \sum_{i=1}^n \sum_{k=1}^K \frac{p_i}{2K} \left(\frac{1}{2} d_{\Hcal \Delta \Hcal} (\Dcal_{i,k}, \Dcal) + \lambda_{i,k}\right) \\
            &\quad + \sqrt{\frac{\log (2nK/\delta)}{2m}}
    \end{aligned}
\end{equation}
where $ \lambda_{i,k} := \inf_{h \in \Hcal} \curlybracket{\Lcal_{\Dcal_{i,k}} (h) + \Lcal_{\Dcal}(h)}$ and $\hat{\Dcal}_{i,k}$ is the observable dataset of the distribution $\Dcal_{i,k}$.
\end{theorem}

We provide a proof in~\Cref{appendix:proof_of_theorem}.
In essence, the theorem establishes an upper bound on the loss of the global model by the weighted sum of: 
i) the risk of each cohort model on its clients' data; 
ii) the discrepancy between the cohort's client distribution and the target distribution, quantified by the second term on the right-hand side; and 
iii) a constant term contingent on the number of data points and the number of sources. 
Additionally, it highlights the benefits of using multiple small cohorts, which generally achieve better performance on their clients’ data compared to larger cohorts. By improving the performance of individual cohort models, we consequently strengthen the performance bound of the global model.
Our theorem extends previous works \cite{taolin,gong2022preserving} to a more general setting where we distill knowledge from teacher models which themselves are trained on a mixture of their respective client distributions.

\section{Evaluation}
\label{sec:eval}

We implement \sys and explore different trade-offs between achieved accuracy, training time, and resource usage for different data distributions and number of cohorts.

\subsection{Experiment setup}
\label{sec:exp_setup}
We have implemented \sys using PyTorch and published all source code online.\footnote{See \url{https://github.com/sacs-epfl/cpfl}.}

\textbf{Dataset.}
We experiment with both the \cifar and \femnist datasets.
\cifar is a common and well-known baseline in machine learning and consists of \num{50000} color images, divided amongst ten classes~\cite{krizhevsky2009learning}.
\femnist contains \num{805263} grayscale images, divided amongst 62 classes~\cite{caldas2018leaf}.
As backbone models, we use a LeNet architecture~\cite{lecun1989handwritten} for \cifar and a CNN~\cite{mcmahan2017communication} for \femnist.
We fix the batch size to 20.
Clients within a cohort train using the standard \fedavg algorithm utilizing the SGD optimizer with a learning rate of $ \eta = 0.002 $ and a momentum of 0.9 for \cifar, and $ \eta = 0.004 $ for \femnist.
These values are taken from existing works~\cite{hsieh2020non,dhasade2023decentralized}.
Clients perform one local epoch in a given round for \cifar and five local steps for \femnist. 

We consider a network with \num{200} nodes for \cifar and take a fixed random subset of \num{1000} nodes (out of \num{3550} nodes) for \femnist.
We vary the number of cohorts $n \in [1,200]$ for \cifar and $n \in [1,64]$ for \femnist where for each setting of $n$, the nodes are split randomly into $n$ cohorts.
For \cifar, we experiment with varying degrees of \ac{non-IID}ness, controlled by $ \alpha $.
Specifically, we construct heterogeneous data splits using the Dirichlet distribution, in line with previous work~\cite{hsu2019measuring,taolin,gong2022preserving}.
We experiment with $ \alpha = 1 $ (low heterogeneity), $\alpha = 0.3$ (moderate heterogeneity) and $\alpha = 0.1 $ (high heterogeneity).
\femnist follows a natural non-IID partitioning based on the data sample creators.

\textbf{Distillation.}
For \cifar, we perform \ac{KD} using the STL-10 dataset~\cite{coates2011analysis}, a dataset inspired by \cifar that has also been utilized in earlier studies~\cite{sattlercfd}.
This dataset contains \num{100000} data samples.
For \femnist, we use the SVHN dataset, containing \num{531131} unlabeled images of house numbers~\cite{sermanet2012convolutional}.
For distillation, we use the Adam optimizer, a learning rate of 0.001, a batch size of 512, and train for 50 epochs.

\textbf{Validation set.}
To progress to the second stage of the algorithm, \sys requires a signal to determine when cohorts have finished training.
To achieve this, we have nodes within the cohort compute and report the loss of the global model on a local validation set.
This validation set is 10\% of the local dataset, and only nodes with at least 10 data samples construct this validation set and report their validation loss.
The cohort server collects all validation losses during each round and averages them.
A cohort stops model training once the minimum validation loss has not decreased for $ r $ rounds.
We have conducted various trials, and we found that a value of $r = 50 $ for \cifar and $ r = 200 $ for \femnist achieves a reasonable trade-off between waiting too long and letting cohort models sufficiently converge.
We also apply a moving average with a window size of 20 to reduce noise.

\textbf{Traces.}
We evaluate \sys under realistic settings to provide a closer approximation to real-world conditions.
To this end, we integrate network and compute capacities traces to simulate the hardware performance of nodes~\cite{lai2022fedscale}.
These traces contain the hardware profile of \num{131000} mobile devices and are sourced initially from the AI and MobiPerf  benchmarks~\cite{ignatov2019ai,huang2011mobiperf}.
They span a broad spectrum ranging from the network speeds of 130 KB/s to 26 MB/s while the compute speeds vary from 0.9 secs to 11.9 secs to train a mini-batch.
We assume that the cohort servers have sufficiently large incoming and outgoing bandwidth capacities that can handle all transfers in a round in parallel.
Finally, all nodes remain online and available during the experiment.

\textbf{Compute infrastructure and implementation.}
We run all experiments on machines in our compute cluster.
Each machine is equipped with dual 24-core AMD EPYC-2 CPU, 128 GB of memory, a NVIDIA RTX A4000 GPU, and is running CentOS 8.
For reproducibility and in line with related work in the domain, we simulate the passing of time in our experiments~\cite{lai2022fedscale,lai2021oort,abdelmoniem2023refl}.
We achieve this by customizing the default event loop provided by the \texttt{asyncio} library and processing events without delay.

We experiment with \fedavg~\cite{mcmahan2017communication} as the default \FL algorithm for training within the cohort.
We report top-1 test accuracies when evaluating the final model with the test sets.
Lastly, we run each experiment with five different seeds, varying data distribution (for \cifar) and node characteristics, and average all results.

\begin{figure*}[t]
	\inputplot{plots/cifar10_no03}{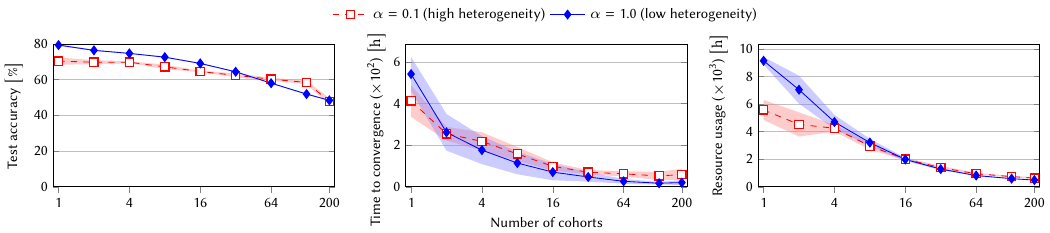}
	\caption{The test accuracy, convergence time and resource usage (in CPU hours) of \cifar, for increasing number of cohorts ($ n $) and different heterogeneity levels (controlled by $ \alpha $). 
    Results for $ \alpha = 0.3 $ are included in~\Cref{sec:app_cifar10}.
    }
\label{fig:cifar10_xlog}
\end{figure*}

\begin{figure*}[t]
	\inputplot{plots/femnist}{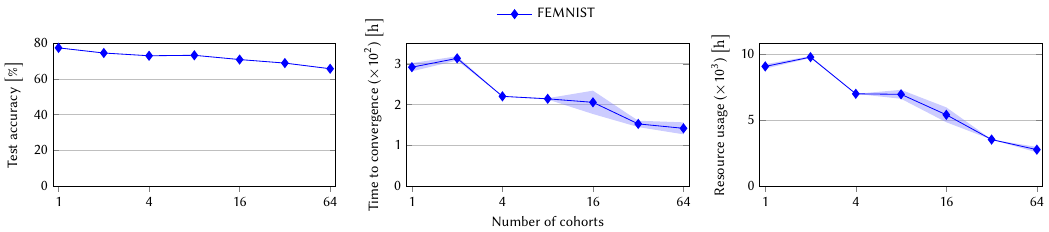}
    \caption{The test accuracy, convergence time and resource usage (in CPU hours) of \femnist, for increasing number of cohorts ($ n $).}
	\label{fig:exp_femnist}
\end{figure*}

\subsection{Time and resource savings by \sys}
\label{sec:exp_performance}
We first evaluate the test accuracy, time to convergence, and resource usage by \sys, for \cifar and \femnist under different data distributions ($\alpha$) and number of cohorts ($n$).
For \cifar, we vary $ n $ from 1 to 200; $ n = 1$ corresponds to the traditional FL setting, whereas $ n = 200 $ assigns each node to its own cohort, corresponding to one-shot \fedkd as described in \Cref{sec:rw}.
We remark that one-shot \fedkd has only been evaluated with 20 nodes at most, and its performance in larger networks remains unknown.
For \cifar and \femnist, we set the client participation rate to 100\% and 20\% respectively.

\textbf{Test accuracy.}
\Cref{fig:cifar10_xlog,fig:exp_femnist} (left plot) show the test accuracy (in \%) of \sys for the \cifar and \femnist datasets respectively, across different values of $n$.
Both datasets show a decreasing trend in test accuracy as $ n $ increases.
This is because (i) each cohort learns on fewer data as $ n $ increases, and (ii) \ac{KD} does not perfectly distill all knowledge.
For \femnist, $ n = 1$ shows a test accuracy of 77.4\%, which gradually decreases to 65.7\% for $ n = 64 $.
For \cifar with $\alpha=0.1$, accuracy drops from $70.4\%$ $(n = 1)$ to $47.9\% $ $(n = 200)$. 
Nonetheless, we observe the drops are marginal when considering lower values of $n$. 
For instance, for \cifar with $\alpha=0.1$, the accuracy of $70.4\%$ at $n=1$ reduces just to $69.8\%$ at $n=4$. Thus lower values of $n$ make for an interesting case for obtaining significant time and resource savings as we show next.

\textbf{Time-to-accuracy.} \Cref{fig:cifar10_xlog,fig:exp_femnist} (middle plot) show the time to convergence of \sys for the same evaluated values of $ n $.
This is the time between starting model training and training completion by all cohorts.
For \cifar, we notice a stark decrease in time to convergence as $ n $ increases: for $ \alpha = 0.1 $ and when increasing $ n $ from 1 to 4, the time until convergence decreases from \SI{413}{\hour} to \SI{218}{\hour}, a speedup of $1.9\times$ with a minimal loss in test accuracy.
Similarly, on the \femnist dataset, the time until convergence reduces from \SI{290}{\hour} $(n = 1)$ to \SI{220}{h} $(n = 4)$, with a speedup of $1.3\times$.
This speedup is because cohorts with fewer data samples require less time to converge.
We remark that~\Cref{fig:cifar10_xlog} excludes the time to complete the KD, but we found this to be minimal compared to the overall training time. 
We provide more details on this in~\Cref{sec:exp_kd_time}.
For \cifar with non-IID data ($\alpha=0.1$), \sys can obtain speedups between $1.6-7\times$ as $n$ varies from 2 to 200.
These speedups might be particularly beneficial in practical FL settings where one has to execute many runs, \eg for hyperparameter tuning~\cite{khodak2021federated}.

We further analyze the training speedups by \sys against \ac{FL} by distilling the models produced by each cohort after a fixed duration. 
Specifically, we distill $10$ and $15$ hours into the experiment for $\alpha = 0.1$ and the \cifar dataset.
10 and 15 hours into the experiment, \ac{FL} attained $28.7$\% and $35.4$\% test accuracy, respectively.
For $ n = 2 $, these accuracies are $36.16$\% and $39.56$\%, respectively.
For $ n = 16 $ the increase in accuracy over \ac{FL} is even more pronounced, reaching $54.1$\% and $56.7$\% accuracy $10$ and $15$ hours into the experiment, respectively.
We attribute these speedups primarily to the decrease in round durations.
As $ n $ increases, cohorts become smaller and thus a round requires fewer client updates compared to the setting with $ n = 1 $ where the server needs all $200$ model updates to conclude a round.

\begin{figure}[t]
	\inputplot{plots/times}{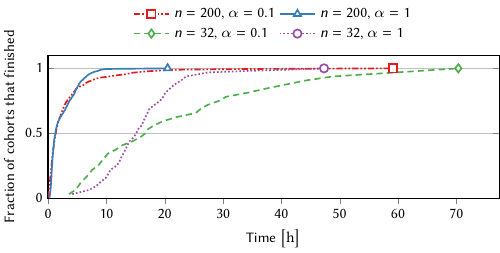}
	\caption{The finish times of individual cohorts, for different numbers of total cohorts and data distributions. We mark the finish time of each group with a symbol.
}
	\label{fig:times}
\end{figure}

\textbf{Resource usage.} \Cref{fig:cifar10_xlog,fig:exp_femnist} (right plot) visualize the resource usage of \sys regarding CPU hours, representing the total time CPUs spent on training. 
We observe a reduction in resource usage thanks to the faster convergence of cohort models. 
For \cifar in the high heterogeneity setting ($\alpha = 0.1$), increasing $n$ from 1 to 4 decreases the required CPU hours from \SI{5577}{\hour} to \SI{4249}{\hour}, or $1.3\times$. 
Similarly for the \femnist dataset, the CPU hours reduce from \SI{9093}{\hour} $(n = 1)$ to \SI{7011}{\hour} $(n = 4)$, or $1.3\times$.
On the other hand, when increasing $n$ up to 200, a substantial reduction of 8.8$\times$ is achieved for $\alpha = 0.1$ for \cifar, albeit with a trade-off of decreased accuracy. 
Consequently, in the non-IID case with $\alpha=0.1$, \sys demonstrates the potential to reduce resource usage by $1.2 - 8.8\times$ as $n$ varies from 2 to 200.
We also observed reductions in communication volume, both for \cifar and \femnist.
They follow a similar trend as the reduction in resource usage, and we further comment on this in \Cref{sec:exp_comm_cost_savings}.

\textbf{Experimental conclusion.} Our results suggest that a reasonable value for the number of cohorts lies between 4 and 16. We conclude that while \ac{FL} ($n = 1$) trains more accurate models compared to \sys, it incurs longer training times and more resource usage. 
On the other hand, one-shot \fedkd $(n = M)$ suffers from accuracy loss but shows great potential to reduce training time and resource usage.
\sys can proficiently navigate these two extremes by controlling the number of cohorts, offering \ac{FL} practitioners a simple way to tailor \ac{FL} training sessions according to their requirements.

The remaining experiments in this section use the \cifar dataset as it allows experimentation with different degrees of non-IIDness.

\subsection{Training time of cohorts}
\label{sec:training_time}
We further analyze the convergence times of cohorts in some of the experiments in~\Cref{sec:exp_performance}.
We plot in~\Cref{fig:times}, for $ n = 32 $ and $ n = 200$, and $ \alpha = 1 $ and $ \alpha = 0.1 $, an ECDF with the fraction of cohorts that completed training as the experiment progressed.
These numbers correspond to the experiment described in~\Cref{sec:exp_performance} on the \cifar dataset.
In the plot, we annotate the completion of the last cohort in each setting with a marker.

Our first observation is that for $ n = 200 $, 75\% of the cohorts converge within less than 5 hours.
We do notice some slower cohorts that prolong overall training.
Compared to when using $ \alpha = 1$, this slowdown is more pronounced for $ \alpha = 0.1 $, where data distribution is much more skewed.
We also notice similar effects for $ n = 32 $, and we observe a higher variation of training times across cohorts for $ n = 32 $ and $ \alpha = 0.1 $, compared to $ n=32 $ and $ \alpha = 1 $.
Our results in ~\Cref{fig:times} suggest that FL practitioners can further gain speedups by proceeding to the \ac{KD} step even when a fraction of cohorts (\eg 75\%) have converged instead of waiting for all cohorts to finish training. 
This is similar in spirit to how the federated server in traditional \ac{FL} tolerates partial updates when not all clients complete the expected number of local steps in specified time~\cite{li2020federated}, albeit at the cost of accuracy.

\begin{figure}
	\inputplot{plots/samples_time}{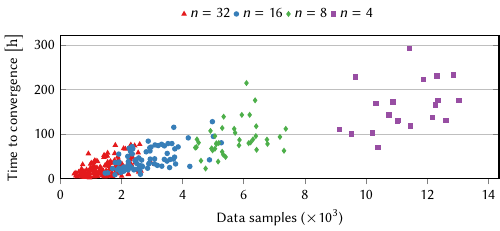}
	\caption{The relation between the number of data samples and the time until convergence within a cohort, for $ n = 4, 8, 16 $ and 32, and under a non-IID data distribution ($ \alpha = 0.1 $). We have $5*n$ measurements for each $n$ due to $5$ seeds.
	}
	\label{fig:data_samples_convergence_time}
\end{figure}

\subsection{Cohort data samples and training time}
\label{sec:data_samples_vs_train_time}

\sys is based on the premise that smaller cohorts \ie cohorts with fewer clients and consequently with fewer total data samples, converge quicker.
To validate this premise, we plot in~\Cref{fig:data_samples_convergence_time} the relation between the number of total data samples and the time until convergence for every cohort under a non-IID data distribution ($ \alpha = 0.1 $).
We extract results for $ n = 4, 8, 16 $ and 32 for the experiment runs described in Section~\ref{sec:exp_performance}.
\Cref{fig:data_samples_convergence_time} hints at a positive relation between the number of data samples in a cohort and the time until convergence.
Therefore, \ac{FL} practitioners can manipulate $n$ to increase or decrease the average number of data samples per cohort, subsequently controlling the time required for the cohort to finish training.

\begin{table}[t]
    \centering
    \begin{tabular}{c|c|c|c|c}
    	\toprule
        \textbf{$\alpha$} & $n$ & \textbf{Teacher acc.} & \textbf{Student acc.} & \textbf{$\Delta$} \\ \midrule
        \multirow{4}{*}{0.1} & 4 & 65.76 $\pm$ 2.63 & 69.79 $\pm$ 0.55 & +4.03 \\
        & 16 & 48.24 $\pm$ 5.86 & 64.59 $\pm$ 0.39 & +16.35 \\ 
        & 64 & 27.64 $\pm$ 7.38 & 60.28 $\pm$ 0.80 & +32.65 \\ 
        & 200 & 17.32 $\pm$ 5.67 & 47.89 $\pm$ 1.75 & +30.57 \\ \midrule
        \multirow{4}{*}{1} & 4 & 74.44 $\pm$ 0.86 & 74.74 $\pm$ 0.30 & +0.30 \\
        & 16 & 63.37 $\pm$ 1.76 & 69.11 $\pm$ 0.34 & +5.74 \\ 
        & 64 & 42.44 $\pm$ 5.17 & 58.05 $\pm$ 0.24 & +15.61 \\ 
        & 200 & 27.00 $\pm$ 5.29 & 48.43 $\pm$ 0.27 & +21.44 \\
        \bottomrule
    \end{tabular}
    \caption{The average accuracy of teacher and student models for varying values of $\alpha$ and $ n $, and for the \cifar dataset. The right-most column shows the improvement in test accuracy by knowledge distillation.}
    \label{tab:kd_performance}
\end{table}

\subsection{Teacher and student accuracies of \ac{KD}}
\label{sec:acc_of_KD}
We now assess the accuracy improvement by \ac{KD}.
We show in~\Cref{tab:kd_performance} the average test accuracy of the teacher and student models and the average improvement $(\Delta)$, for $ \alpha = 0.1 $ and $ \alpha = 1$ across four different values of $ n $.
We also report the corresponding standard deviations.

\Cref{tab:kd_performance} shows that increasing $ n $ improves the gain in test accuracy induced by \ac{KD}.
This is because, with higher values of $ n $, each cohort trains on proportionately less data, affecting the generalization performance of the teacher models.
\ac{KD} in this case draws a larger \textit{relative} improvement.
We also observe that the gains in accuracy are higher for $\alpha = 0.1$ compared to $\alpha = 1.0$.
We suspect this is because under high heterogeneity, \ac{KD} is more capable of combining models that recognize specialized features compared to low heterogeneity settings where mostly similar models are produced.
This observation also aligns with \ac{KD}'s original concept as a method to effectively combine knowledge from \textit{specialized} ensembles, shown in ~\cite{hintondistillation}.

\section{Conclusion}
This paper introduced Cohort-Parallel Federated Learning (\sys), an innovative approach to enhance \ac{FL} by partitioning the network into several smaller cohorts.
The underpinning principle of \sys is that smaller networks lead to quicker convergence and more efficient resource utilization.
By unifying the cohort models with Knowledge Distillation and a cross-domain, unlabeled dataset, we produce a final global model that integrates the contributions from all clients.
Our experimental findings confirm that this strategy yields significant advantages in training time and resource utilization without considerably compromising test accuracy.
This approach offers practitioners a tangible means to control their FL resource consumption and convergence timelines.

\begin{acks}
This work has been funded by the Swiss National Science Foundation, under the project ``FRIDAY: Frugal, Privacy-Aware and Practical Decentralized Learning'', SNSF proposal No. 10.001.796.
\end{acks}

\bibliographystyle{plain}
\bibliography{main.bib}

\clearpage

\appendix

\section{Theoretical Analysis}
\label{appendix:proof_of_theorem}
We begin by recalling the notations and definitions used throughout the theoretical analysis. We define $\mathcal{H} \Delta \mathcal{H}$-divergence between a source $\Dcal'$ and target $\Dcal$ distribution as:
\begin{align*}
    d_{\mathcal{H} \Delta \mathcal{H}} (\Dcal', \Dcal) 
    &= 2 \sup_{h,h' \in \mathcal{H}} 
    \bigg| \mathbb{P}_{x \sim \Dcal'} (h(x) \neq h'(x)) \\
    &\quad - \mathbb{P}_{x \sim \Dcal} (h(x) \neq h'(x)) \bigg|
\end{align*}
Let $h_s = \sum_{i=1}^n p_i h_i$ where $\sum_{i=1}^n p_i = 1$ and $p_i > 0$, denote the hypothesis on the global sever and $h_i$ the one on the cohort $i$. Let $\Dcal$ be any target distribution and $\Dcal_i$ be a mixture of $K$'s local distribution $\{\Dcal_{i,k}\}_{k=1}^K$, that is, $\Dcal_i = \frac{1}{K}\sum_{k=1}^K \Dcal_{i,k}$. Denote the expected risk of hypothesis $h$ on the distribution $\Dcal$ as $\mathcal{L}_{\Dcal}(h)$. We state the following lemmas.

\begin{lemma}[Theorem 3, \cite{domain-adaptation-theory}]
\label{lmm:domain-adapt}
For any target distribution $\Dcal$ and source distribution $\Dcal'$. The expected risk of a hypothesis $h$ on the $\Dcal$ is bounded by
\begin{align}
    \Lcal_{\Dcal} (h) \leq \Lcal_{\Dcal'} (h) + \frac{1}{2} d_{\Hcal \Delta \Hcal} (\Dcal',\Dcal) + \lambda 
\end{align}
where we note $\lambda = \inf_{h \in \Hcal} \curlybracket{\Lcal_{\Dcal'}(h) + \Lcal_{\Dcal}(h)}$
\end{lemma}

\begin{lemma}
\label{lmm:mixture-discrepancy}
Let $\{\Dcal_k\}_{k=1}^K$ be a set of distributions and define $\Dcal' = \frac{1}{K} \sum_{k=1}^K \Dcal_k$. For any distribution $\Dcal$, the $\Hcal \Delta \Hcal$-divergence between $\Dcal'$ and $\Dcal$ is bounded by 
\begin{align}
    d_{\Hcal \Delta \Hcal} (\Dcal', \Dcal) \leq \frac{1}{K} \sum_{k=1}^K d_{\Hcal \Delta \Hcal} (\Dcal_k, \Dcal) 
\end{align}
\end{lemma}
\begin{proof}
    \begin{align}
    &d_{\Hcal \Delta \Hcal} (\Dcal', \Dcal) \\
    =& 2 \sup_{h,h' \in \mathcal{H}} \bigg| 
        \mathbb{P}_{x \sim \Dcal'} (h(x) \neq h'(x)) 
    - \mathbb{P}_{x \sim \Dcal} (h(x) \neq h'(x)) \bigg| \nonumber \\
    =& 2 \sup_{h,h' \in \Hcal} \bigg|
        \mathbb{P}_{x \sim \frac{1}{K} \sum_{k=1}^K\Dcal_{k}} (h(x) \neq h'(x)) 
    - \mathbb{P}_{x \sim \Dcal} (h(x) \neq h'(x)) \bigg| \nonumber \\
    =& 2 \sup_{h,h' \in \Hcal} \bigg|
        \frac{1}{K} \sum_{k=1}^K \mathbb{P}_{x \sim \Dcal_k} (h(x) \neq h'(x)) 
    - \mathbb{P}_{x \sim \Dcal} (h(x) \neq h'(x)) \bigg| \nonumber \\
    \leq& \frac{2}{K} \sum_{k=1}^K \sup_{h,h' \in \Hcal} \bigg|
        \mathbb{P}_{x \sim \Dcal_k} (h(x) \neq h'(x)) 
    - \mathbb{P}_{x \sim \Dcal} (h(x) \neq h'(x)) \bigg| \nonumber \\
    =& \frac{1}{K} \sum_{k=1}^K d_{\Hcal \Delta \Hcal} (\Dcal_k, \Dcal)
\end{align}
\end{proof}
\setcounter{theorem}{0}
\begin{theorem}
Let $\Hcal$ be a finite hypothesis class and $h_s :=\sum_{i=1}^n p_i h_i$, where $p_i > 0$ and $\sum_{i=1}^n p_i = 1$. Suppose that each source dataset has $m$ instances. Then, for any $\delta \in (0,1)$, with probability at least $1-\delta$, the expected risk of $h_s$ on the target distribution $\Dcal$ is bounded by:
\begin{align}
    \Lcal_{\Dcal} (h_s) 
    &\leq \sum_{i=1}^n \sum_{k=1}^K \frac{p_i}{K} \Lcal_{\hat{\Dcal}_{i,k}} (h_{i}) \nonumber \\
    &\quad + \sum_{i=1}^n \sum_{k=1}^K \frac{p_i}{K} \bigg( \frac{1}{2} d_{\Hcal \Delta \Hcal} (\Dcal_{i,k}, \Dcal) + \lambda_{i,k} \bigg) \nonumber \\
    &\quad + \sqrt{\frac{\log (2nK/\delta)}{2m}}
\end{align}
\end{theorem}
\begin{proof}
By the definition of $h_s$ and the Jensen inequality, we have:
    \begin{align}
    \label{eq:avg-hypo}
    \Lcal_{\mathcal{D}} (h_s) 
    := \Lcal_{\mathcal{D}} \parenthese{\sum_{i=1}^n p_i h_i} 
    \leq \sum_{i=1}^n p_i \Lcal_{\mathcal{D}}(h_i)
\end{align}
Using \Cref{lmm:domain-adapt}, the expected risk of $h_i$ on $\Dcal$ is bounded by:
\begin{align}
\label{eq:app_lmm1}
    \Lcal_{\Dcal} (h_i) 
    & \leq \Lcal_{\Dcal_i} (h_i) + \frac{1}{2} d_{\Hcal \Delta \Hcal} (\Dcal_i, \Dcal) + \lambda_i
\end{align}
where $\lambda_i = \inf_{h \in \Hcal} \curlybracket{\Lcal_{\Dcal}(h) + \Lcal_{\Dcal_i} (h)}$. Let $\Dcal_i = \frac{1}{K} \sum_{k=1}^K \Dcal_{i,k}$ be a mixture of distributions in cohort $i$, we have then:
\begin{align}
\label{eq:cohort-hypothesis}
    &\Lcal_{\Dcal} (h_i) 
    \leq \Lcal_{\Dcal_i} (h_i) + \frac{1}{2} d_{\Hcal \Delta \Hcal} (\Dcal_i, \Dcal) + \lambda_i \nonumber \\
    \leq & \Lcal_{\Dcal_i} (h_i) + \frac{1}{2K} \sum_{k=1}^K d_{\Hcal \Delta \Hcal} (\Dcal_{i,k}, \Dcal) + \lambda_i \nonumber \\
    \leq & \frac{1}{K} \sum_{k=1}^{K} \Lcal_{\Dcal_{i,k}} (h_{i}) 
        + \frac{1}{2K} \sum_{k=1}^K d_{\Hcal \Delta \Hcal} (\Dcal_{i,k}, \Dcal) 
        + \lambda_i \nonumber \\
    \leq & \frac{1}{K} \sum_{k=1}^{K} \Lcal_{\Dcal_{i,k}} (h_{i}) 
        + \frac{1}{2K} \sum_{k=1}^K d_{\Hcal \Delta \Hcal} (\Dcal_{i,k}, \Dcal) 
        + \frac{1}{K} \sum_{k=1}^K \lambda_{i,k} 
\end{align}
where the second inequality is application of \Cref{lmm:mixture-discrepancy}. The third one follows the fact that $\Dcal_i = \frac{1}{K} \sum_{k=1}^K \Dcal_{i,k}$. The last inequality follows the same fact and application of triangle and Jensen's inequality on $\lambda_i$. Using Hoeffding's inequality, with probability $1-\frac{\delta}{nK}$, the risk of hypothesis $h_i$ on the source distribution $\Dcal_{i,k}$ is upper bounded by 
\begin{align}
\label{eq:hoeffding}
    \Lcal_{\Dcal_{i,k}} (h_i) \leq \Lcal_{\hat{\Dcal}_{i,k}} (h_i) + \sqrt{\frac{\log (2nK/\delta)}{2m}}
\end{align}
where $\Lcal_{\hat{\Dcal}_{i,k}} (h_i)$ is the risk of $h_i$ on the empirical distribution $\hat{\Dcal}_{i,k}$. Combining \Cref{eq:avg-hypo,eq:cohort-hypothesis,eq:hoeffding} and using the same analysis as in \cite{taolin}. With probability at least $1-\delta$ over $nK$ sources of $m$ samples, we have : 
\begin{align}
    &\Lcal_{\Dcal} (h_s) \nonumber \\
    \leq& \sum_{i=1}^n \sum_{k=1}^K \frac{p_i}{K} \Lcal_{\Dcal_{i,k}} (h_{i}) 
        + \sum_{i=1}^n \sum_{k=1}^K \frac{p_i}{2K}  d_{\Hcal \Delta \Hcal} (\Dcal_{i,k}, \Dcal) \nonumber \\
        &\quad + \sum_{i=1}^n \sum_{k=1}^K \frac{p_i}{K}  \lambda_{i,k} \nonumber \\
    \leq& \sum_{i=1}^n \sum_{k=1}^K \frac{p_i}{K} \Lcal_{\hat{\Dcal}_{i,k}} (h_{i}) 
        + \sum_{i=1}^n \sum_{k=1}^K \frac{p_i}{K} \sqrt{\frac{\log (2nK/\delta)}{2m}} \nonumber \\
    &\quad + \sum_{i=1}^n \sum_{k=1}^K \frac{p_i}{2K}  d_{\Hcal \Delta \Hcal} (\Dcal_{i,k}, \Dcal) 
        + \sum_{i=1}^n \sum_{k=1}^K \frac{p_i}{K}  \lambda_{i,k} \nonumber \\
    \leq&  \sum_{i=1}^n \sum_{k=1}^K \frac{p_i}{K} \Lcal_{\hat{\Dcal}_{i,k}} (h_{i}) 
        + \sum_{i=1}^n \sum_{k=1}^K \frac{p_i}{K} \bigg( 
        \frac{1}{2} d_{\Hcal \Delta \Hcal} (\Dcal_{i,k}, \Dcal) + \lambda_{i,k} 
    \bigg) \nonumber \\
    &\quad + \sqrt{\frac{\log (2nK/\delta)}{2m}}
\end{align}

\end{proof}

\section{Additional Notes on Experimental Evaluation}
\label{sec:app_exp_setup}

\subsection{Motivational Plot (\Cref{fig:motivation})}
\label{sec:setup_motivation_plot}

\Cref{fig:motivation}, showing the evolution of validation loss and motivating the approach behind \sys, has been generated from the \cifar experiments described in~\Cref{sec:exp_performance}.
These plots correspond to the validation loss for each of the four cohorts, in a setting with $ n = 4 $ and when using 90 as seed, for $ \alpha = 1 $ and $ \alpha = 0.3 $.

In \Cref{fig:motivation}, we chose to highlight the results of a single run rather than averaging across multiple runs. The rationale behind this decision stems from the fact that when averaging, the stopping criterion which applies to individual runs does not hold any meaningful interpretation for the averaged result.
This means that an average might misrepresent the number of rounds required for convergence or even the validity of the stopping criterion itself.
By focusing on a single run, we ensure the integrity and applicability of our stopping criterion, providing a clearer and more direct interpretation of our results.
We have manually verified that our conclusion also hold true for different seeds and number of cohorts.

\subsection{Knowledge Distillation}
\label{sec:exp_kd_time}
\Cref{fig:cifar10_xlog} excludes the time it takes for the \ac{KD} process to complete. For \cifar, we noticed that \ac{KD} takes between 50 minutes (for \( n = 2 \)) and 305 minutes (for \( n = 200 \)) to complete. 
This time frame represents a small fraction of the time and resources required for model training by cohorts. 
Regarding \femnist, the \ac{KD} process takes between 59 minutes (for \( n = 2 \)) and 16.8 hours (for \( n = 64 \)). 
In the case of \femnist, the majority of time is spent on generating inferences from the teacher models. 
We propose two methods to expedite this process. First, one can parallelize this process by generating logits from distinct teacher models simultaneously. 
Second, one can increase the inference batch size, although at the cost of additional memory usage.

\subsection{\cifar Results with $ \alpha =0.3 $}
\label{sec:app_cifar10}

We show in~\Cref{fig:cifar10_a0.3} the test accuracy, convergence time and resource usage (CPU hours) of the \cifar dataset.
Compared to~\Cref{fig:cifar10_xlog}, this plot includes the $ \alpha = 0.3 $ setting that we omitted from~\Cref{fig:cifar10_xlog} for presentation clarity.

\begin{figure*}[t]
	\inputplot{plots/cifar10_page_xlog}{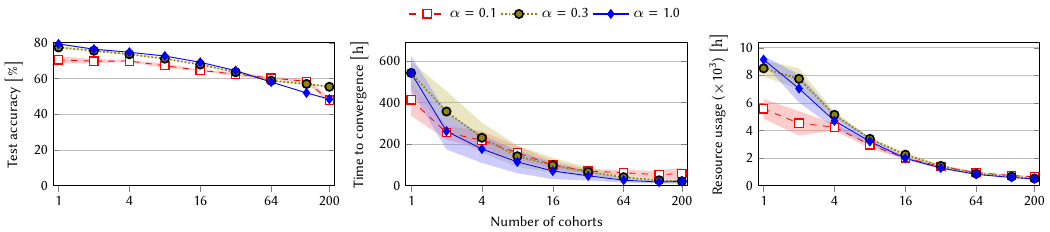}
	\caption{The test accuracy, convergence time and resource usage (in CPU hours) of \cifar, for different number of cohorts ($n$) and different heterogeneity levels (controlled by $ \alpha $).
    }
	\label{fig:cifar10_a0.3}
\end{figure*}

\subsection{Savings in Communication Volume by \sys}
\label{sec:exp_comm_cost_savings}
Besides savings in training time and CPU resource usage (see~\Cref{fig:cifar10_xlog} and~\ref{fig:exp_femnist}), \sys also provides savings in communication volume.
\Cref{fig:comm_costs} shows the communication volume required by \sys, for the \cifar and \femnist dataset, and for different number of cohorts ($ n $).
The trend in communication volume as $ n $ increases is comparable to the trend in time to convergence and resource usage shown in~\Cref{fig:cifar10_xlog,fig:exp_femnist}.
For \cifar with $ n = 200 $, since each node will perform standalone training of its model, we remark that the only communication volume incurred is when sending the trained cohort model to the global server.

For the same value of $ n $, we notice that \femnist incurs significantly more communication volume compared to \cifar.
This is because the model size of \femnist in serialized form is significantly larger than that of \cifar, namely 6.7 MB for \femnist compared to 346 KB for \cifar.
Additionally, \femnist also requires many more rounds to converge than \cifar due to the large overall network size (1000 nodes) and the challenging nature of the 62-class classification task. 
As a result, this leads to escalated communication costs.

\begin{figure*}[t]
	\centering
	\inputplot{plots/comm_costs}{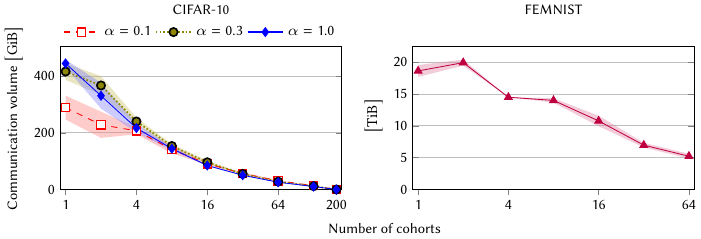}
	\caption{The communication volume required by \sys for convergence, for \cifar and \femnist, and for different number of cohorts ($ n $). For \cifar we also show the communication volume for different values of $\alpha$.}
	 \label{fig:comm_costs}
\end{figure*}

\end{document}